\documentclass[preprint]{elsarticle}
\usepackage{times}
\usepackage{adjustbox}
\usepackage{tabularx}
\usepackage{amsmath}
\usepackage{multirow}
\usepackage{hyperref}

\usepackage{comment}
\usepackage{amsthm}
\usepackage{blindtext}
\usepackage{stix}
\usepackage{float}
\usepackage{enumitem}
\usepackage{multirow}
\usepackage{graphicx}
\usepackage{algorithm}
\usepackage{algorithmic}
\usepackage{amsfonts}
\usepackage{amsthm}

\newtheorem{definition}{Definition}
\newtheorem{theorem}{Theorem}
\newtheorem{corollary}{Corollary}
\newtheorem{lemma}{Lemma}

\usepackage{color}
\journal{Journal of \LaTeX\ Templates}

\begin{document}
\begin{frontmatter}

\title{Differentially Private High Dimensional Sparse Covariance Matrix Estimation\tnoteref{t1,t2}}
\tnotetext[t1]{A preliminary version appeared in Proceedings of The 53rd Annual Conference on Information Sciences and Systems (CISS 2019).}
\tnotetext[t2]{This research was supported in part by the National Science Foundation (NSF) through grants CCF-1422324 and CCF-1716400.}
\author[1]{Di Wang \corref{cor1}}
\ead{dwang45@buffalo.edu}
\author[1]{Jinhui Xu}
\ead{jinhui@buffalo.edu}
\address[1]{Department of Computer Science and Engineering\\
State University of New York at Buffalo\\
338 Davis Hall, Buffalo, 14260}
\cortext[cor1]{Corresponding author}
\begin{abstract}
In this paper, we study the problem of estimating the covariance matrix under differential privacy, where the underlying covariance matrix is assumed to be sparse and of high dimensions. We propose a new method, called DP-Thresholding, to achieve a non-trivial $\ell_2$-norm based error bound, 
which is significantly better than the existing ones from adding noise directly to the empirical covariance matrix.  We also extend the $\ell_2$-norm based error bound to a general $\ell_w$-norm based one for any $1\leq w\leq \infty$, and show that they share the same upper bound asymptotically. Our approach can be easily extended to local differential privacy. Experiments on the synthetic datasets show consistent results with our theoretical claims.
\end{abstract}


\end{frontmatter}

\section{Introduction}

Machine Learning and Statistical Estimation have made profound impact in recent years to many applied domains such as social sciences, genomics, and medicine. During their applications, a frequently encountered challenge is how 
to deal with the high dimensionality of the datasets, especially for those in genomics, educational and psychological research. A commonly adopted strategy  for dealing with such an issue is to assume that 
the underlying structures of parameters are sparse.

Another often encountered challenge 
is how to handle sensitive data, such as those  in social science,
 biomedicine and genomics. A promising approach is to use some differentially private mechanisms for the statistical inference and learning tasks. Differential Privacy (DP) \cite{dwork2006calibrating} is a widely-accepted criterion that provides provable protection against identification and is resilient to arbitrary auxiliary information that might be available to attackers.  Since its introduction over a decade ago, a rich line of works are now available, which have made differential privacy a compelling privacy enhancing technology for 
 many organizations, such as Uber \cite{uber}, Google \cite{google}, Apple \cite{apple}.

Estimating or studying the high dimensional datasets while keeping them (locally) differentially private could be quite challenging for many problems, such as sparse linear regression \cite{dwangppml18}, sparse mean estimation \cite{duchi2018right} and selection problem \cite{ullman2018tight}. However, there are also evidences showing that 
the loss of some problems under the privacy constraints can be quite small
compared with their non-private counterparts. Examples of such nature include  
high dimensional sparse PCA \cite{ge2018minimax}, sparse inverse covariance estimation \cite{dwangglobalsip18}, and high-dimensional distributions estimation \cite{kamath2018privately}. Thus, it is desirable to determine which high dimensional problem can be learned or estimated efficiently in a private manner. 

In this paper, we try to give an answer to this question for 
a simple but fundamental problem in machine learning and statistics, called estimating the underlying sparse covariance matrix of bounded sub-Gaussian distribution. For this problem, we propose a simple but nontrivial $(\epsilon,\delta)$-DP method, DP-Thresholding, and show that the squared $\ell_w$-norm error for any $1\leq w\leq \infty$ is bounded by $O(\frac{s^2\log p}{n\epsilon^2})$, where $s$ is the sparsity  of each row in the underlying covariance matrix. Moreover, our method can be easily extended to the local differentialy privacy model. Experiments on synthetic datasets confirm the theoretical claims.  
To our best knowledge, this is the first paper studying the problem of estimating high dimensional sparse covariance matrix under (local) differential privacy.
\section{Related Work}

Recently, there are several papers studying private distribution estimation, such as \cite{kamath2018privately,joseph2018locally,karwa2017finite,gaboardi2018locally,kareemppml18}. For distribution estimation under the central differential privacy model, \cite{karwa2017finite} considers the 1-dimensional private mean estimation of a Gaussian distribution with (un)known variance. The work that is probably most related to ours is \cite{kamath2018privately}, which studies the problem of privately learning a multivariate Gaussian and product distributions. The following are the main differences with ours.
Firstly, our goal is  
to estimate the covariance of a sub-Gaussian distribution. 
Even though the class of distributions considered in our paper is larger than the one in \cite{kamath2018privately}, it has an additional assumption which requires the $\ell_2$ norm of a sample of the distribution to be bounded by $1$. This means that it does not include the general Gaussian distribution. Secondly, although \cite{kamath2018privately} also considers the high dimensional case, it does not assume the sparsity of the underlying covariance matrix. Thus, its error bound depends on the dimensionality $p$ polynomially, which is large in the high dimensional case ($p\gg n$), while the dependence in our paper is only logarithmically ({\em i.e.,} $\log p$). Thirdly, the error in \cite{kamath2018privately} is measured by  the total variation distance, while it is by $\ell_w$-norm in our paper. Thus, the two results are not comparable. Fourthly, the methods in \cite{kamath2018privately} seem difficult to be extended to the local model.  \cite{kareemppml18} recently also studies the covariance matrix estimation via iterative eigenvector sampling. However, their method is just for the low dimensional case and with Frobenious norm as the error measure.

Distribution estimation under local differential privacy has been studied in \cite{gaboardi2018locally,joseph2018locally}. However, both of them study only the 1-dimensional Gaussian distribution. Thus, it is quite different from the class of distributions in our paper.

In this paper, we mainly use Gaussian mechanism to the covariance matrix, which has been studied in \cite{dwork2014analyze,ge2018minimax,dwangglobalsip18}. However, as it will be shown later, simply outputting the perturbed covariance can cause big error and thus
is insufficient for our problem.
Compared to these problems, ours is clearly more complicated.

\section{Preliminaries}
\subsection{Differential Privacy}
	Differential privacy \cite{dwork2006calibrating} is by now a defacto standard for statistical data privacy which constitutes a strong standard for privacy guarantees for algorithms on aggregate databases. One likely reason that it gains so much popularity is its guarantee of no significant change on the outcome distribution when there is one entry change to the dataset.  We say that two datasets $D,D'$ are neighbors if they differ by only one entry, denoted as $D \sim D'$.
\begin{definition}[Differentially Private\cite{dwork2006calibrating}]\label{def:7}
	A randomized algorithm $\mathcal{A}$ is $(\epsilon,\delta)$-differentially private (DP) if for all neighboring datasets $D,D'$ and for all events $S$ in the output space of $\mathcal{A}$, the following holds
	$$\mathbb{P}(\mathcal{A}(D)\in S)\leq e^{\epsilon} \mathbb{P}(\mathcal{A}(D')\in S)+\delta.$$
	When $\delta = 0$,  $\mathcal{A}$ is $\epsilon$-differentially private.
\end{definition}
We will use Gaussian Mechanism \cite{dwork2006calibrating} to guarantee $(\epsilon,\delta)$-DP.
\begin{definition}[Gaussian Mechanism]\label{def:8}
	Given any function $q : \mathcal{X}^n\rightarrow \mathbb{R}^p$, the Gaussian Mechanism is defined as:
	\[ \mathcal{M}_G(D,q,\epsilon)=q(D)+ Y,\]
	where Y is drawn from Gaussian Distribution $\mathcal{N}(0,\sigma^2I_p)$ with  $\sigma\geq \frac{\sqrt{2\ln(1.25/\delta)}\Delta_2(q)}{\epsilon}$. Here $\Delta_2(q)$ is the $\ell_2$-sensitivity of the function $q$, i.e. $$\Delta_2(q)=\sup_{D\sim D'}||q(D)-q(D')||_2.$$
	Gaussian Mechanism preservers $(\epsilon,\delta)$-differential privacy. 
\end{definition}
\subsection{Private Sparse Covariance Estimation}
Let $x_1, x_2, \cdots, x_n$ be $n$ random samples from a $p$-variate distribution with covariance matrix $\Sigma=(\sigma_{ij})_{1\leq i,j\leq p}$, where the dimensionality $p$ is assumed to be high, {\em i.e.,} $p\gg n\geq \text{Poly}(\log p)$. 

We define the parameter space of $s$-sparse covariance matrices as the following:
\begin{equation}\label{eq:1}
    \mathcal{G}_{0}(s)=\{\Sigma=(\sigma_{ij})_{1\leq i,j\leq p}: \sigma_{-j,j} \text{ is } s\text{-sparse } \forall j\in[p] \},
\end{equation}
where $\sigma_{-j,j}$ means the $j$-th column of $\Sigma$ with the entry $\sigma_{jj}$ removed. That is, a matrix in $ \mathcal{G}_{0}(s)$ has at most $s$ non-zero off-diagonal elements in each column.

We  assume that each $x_i$ is sampled from a $0$-mean and sub-Gaussian distribution with parameter $\sigma^2$, that is, 
\begin{equation}\label{eq:2}
  \mathbb{E}[x_i]=0, \mathbb{P}\{|v^Tx_i|>t\}\leq e^{-\frac{t^2}{2\sigma^2}},\forall t>0 \text{ and } \|v\|_2=1.
\end{equation}
This means that all the one-dimensional marginals of $x_i$ have sub-Gaussian tails. We also assume that with probability 1, $\|x_i\|_2\leq 1$. We note that such assumptions are quite common in the differential privacy literature, such as \cite{ge2018minimax}.

Let $\mathcal{P}_d(\sigma^2, s)$ denote the set of distributions of $x_i$ satisfying  all the above conditions ({\em \i.e.,} (\ref{eq:2}) and $\|x_i\|_2\leq 1$) and with the covariance matrix $\Sigma\in \mathcal{G}_0(s)$. The goal of private covariance estimation is to obtain an estimator $\Sigma^{\text{priv}}$ of the underlying covariance matrix $\Sigma$ based on $\{x_1, \cdots, x_n\}\sim P\in \mathcal{P}_d(\sigma^2, s)$ while keeping it differnetially private. In this paper, we will focus on the $(\epsilon, \delta)$-differential privacy. We  use the $\ell_2$ norm to measure the difference between $\Sigma^{\text{priv}}$ and $\Sigma$, {\em i.e.,} $\|\Sigma^{\text{priv}}-\Sigma\|_2$.

\begin{lemma}
Let $\{x_1, \cdots, x_n\}$ be $n$ random variables sampled from Gaussian distribution  $\mathcal{N}(0, \sigma^2)$. Then 
\begin{align}
    &\mathbb{E}{\max_{1\leq i\leq n}|x_i|}\leq \sigma \sqrt{2\log 2n} \label{eq:3},\\
    & \mathbb{P}\{\max_{1\leq i\leq n  }|x_i|\geq t\}\leq 2ne^{-\frac{t^2}{2\sigma^2}} \label{eq:4}.
\end{align}
 Particularly, if $n=1$, we have $\mathbb{P}\{ |x_i|\geq t\}\leq 2e^{-\frac{t^2}{2\sigma^2}}$.
\end{lemma}

\begin{lemma}[\cite{cai2012optimal}]
If $\{x_1, x_2, \cdots, x_n\}$ are sampled form a sub-Gaussian distribution in (\ref{eq:2}) and  $\Sigma^*=(\sigma^*)_{1\leq i,j\leq p}=\frac{1}{n}\sum_{i=1}^nx_ix_i^T$ is the empirical covariance matrix, then there exist constants $C_1$ and $\gamma>0$ such that $\forall i,j\in[p]$
\begin{equation}\label{eq:5}
    \mathbb{P}(|\sigma_{ij}^*-\sigma_{ij}|> t)\leq C_1e^{-nt^2\frac{8}{\gamma^2}}
\end{equation}
for all $|t|\leq \delta$, where $C_1$ and $\gamma$ are constants and depend only on $\sigma^2$. 
Specifically, 
\begin{equation}\label{eq:6}
    \mathbb{P}\{|\sigma_{ij}^*-\sigma_{ij}|>\gamma\sqrt{\frac{\log p}{n}}\}\leq C_1p^{-8}. 
\end{equation}
\end{lemma}
\section{Method}
\subsection{A First Approach}
A direct way to obtain a private estimator is to perturb the empirical covariance matrix by symmetric Gaussian matrices, which has been used in previous work on private PCA, such as \cite{dwork2014analyze,ge2018minimax}. However, as we can see bellow, this method will introduce big error. 

By \cite{dwork2014analyze}, for any give $0<\epsilon, \delta\leq 1$ and $\{x_1,x_2,\cdots, x_n\}\sim P\in \mathcal{P}_p(\sigma^2,s)$, the following perturbing procedure is $(\epsilon, \delta)$-differentially private:
\begin{equation}\label{eq:7}
    \tilde{\Sigma}=\Sigma^*+N=(\tilde{\sigma}_{ij})_{1\leq i,j\leq p}=\frac{1}{n}\sum_{i=1}^nx_ix_i^T+N,
\end{equation}
where $N$ is a symmetric matrix with its upper triangle ( including the diagonal) being i.i.d samples from 
	    $\mathcal{N}(0, \sigma_1^2)$; here $\sigma_1^2=\frac{2\ln(1.25/\delta)}{n^2\epsilon^2}$, and each lower triangle entry is copied from its upper triangle counterpart. By \cite{tao2012topics}, we know that $\|N\|_2\leq O(\sqrt{p}\sigma_1)=O(\frac{\sqrt{p}\sqrt{\log \frac{1}{\delta}}}{n\epsilon})$. We can easily get that 
	    \begin{equation}\label{eq:8}
	        \|\tilde{\Sigma}-\Sigma\|_2\leq \|\Sigma^*-\Sigma\|_2+\|N\|_2\leq  O(\frac{\sqrt{p\log \frac{1}{\delta}}}{n\epsilon}),
	    \end{equation}
    where the second inequality is due to \cite{tropp2015introduction}. However, we can see that the upper bound of the error in (\ref{eq:8}) is quite large in the high dimensional case.
    
    Another issue of the private estimator in (\ref{eq:7})
    is that it is not clear whether it  is positive-semidefinite, a property that is normally expected from an estimator.
    
    
\subsection{Post-processing via Thresholding}

We note that one of the reasons that the private estimator $\tilde{\Sigma}$ in (\ref{eq:7}) fails is due to the fact that some entries are quite large which make $\|\tilde{\Sigma}_{ij}-\Sigma_{ij}\|_2$  large for some $i,j$. To see it more precisely, by (\ref{eq:4}) and (\ref{eq:5}) we can get the following, with probability at least $1-Cp^{-6}$, 
for all $1\leq i,j\leq p$,
\begin{equation}\label{eq:9}
    |\tilde{\sigma}_{ij}-\sigma_{ij}|\leq \gamma\sqrt{\frac{\log p}{n}}+\frac{4\sqrt{2\ln \frac{1.25}{\delta}}\sqrt{\log p}}{n\epsilon}=O(\gamma\sqrt{\frac{\log p}{n\epsilon^2}}).
\end{equation}
Thus, to reduce the error, it is natural to think of the following way. 
For those $\sigma_{ij}$ with larger values, we  keep the corresponding $\tilde{\sigma}_{ij}$ in order to make their difference less than some threshold. For  those $\sigma_{ij}$ with smaller values compared with (\ref{eq:9}), since the corresponding $\tilde{\sigma}_{ij}$ may still be large, if we threshold $\tilde{\sigma}_{ij}$ to 0, we can lower the error on $\tilde{\sigma}_{ij}-\sigma_{ij}$.

Following the above thinking 
and the thresholding methods in \cite{cai2012optimal} and \cite{bickel2008covariance}, we propose the following  DP-Thresholding method,  which post-processes the perturbed covariance matrix in (\ref{eq:7}) with the threshold $\gamma\sqrt{\frac{\log p}{n}}+\frac{4\sqrt{2\ln 1.25/\delta}\sqrt{\log p}}{n\epsilon}$. After thresholding, we further threshold the eigenvalues of $\hat{\Sigma}$ in order to make it positive semi-definite. See Algorithm \ref{alg:1} for detail. 
	\begin{algorithm}[h]
		\caption{DP-Thresholding} 
		$\mathbf{Input}$: $\epsilon, \delta$ are privacy parameters and $\{x_1,x_2,\cdots, x_n\}\sim P\in\mathcal{P}(\sigma^2, s)$.
		\begin{algorithmic}[1]
		   \STATE Compute 
 \begin{equation*}
    \tilde{\Sigma}=(\tilde{\sigma}_{ij})_{1\leq i,j\leq p}=\frac{1}{n}\sum_{i=1}^nx_ix_i^T+N,
\end{equation*}
where $N$ is a symmetric matrix with its upper triangle (including the diagonal) being i.i.d samples from 
	    $\mathcal{N}(0, \sigma_1^2)$; here $\sigma_1^2=\frac{2\ln(1.25/\delta)}{n^2\epsilon^2}$, and each lower triangle entry is copied from its upper triangle counterpart.
	    \STATE Define the thresholding estimator $\hat{\Sigma}=(\hat{\sigma}_{ij})_{1\leq i,j\leq n}$ as
		\begin{equation}\label{eq:10}
		    \hat{\sigma}_{ij}=\tilde{\sigma}_{ij}\cdot I[|\tilde{\sigma}_{ij}|>\gamma\sqrt{\frac{\log p}{n}}+\frac{4\sqrt{2\ln 1.25/\delta}\sqrt{\log p}}{n\epsilon}].
		\end{equation}
		\STATE Let the eigen-decomposition of $\hat{\Sigma}$ as $\hat{\Sigma}=\sum_{i=1}^p\lambda_iv_iv_i^T$. Let $\lambda^+=\max\{\lambda_i, 0\}$ be the positive part of $\lambda_i$, then define $\Sigma^+=\sum_{i=1}^p\lambda^+ v_iv_i^T$.
		\RETURN $\Sigma^+$.
		\end{algorithmic}\label{alg:1}
	\end{algorithm}
	
\begin{theorem}\label{thm:1}
For any $0<\epsilon, \delta\leq 1$, Algorithm \ref{alg:1} is $(\epsilon, \delta)$-differentially private. 
\end{theorem}

\begin{proof}
By \cite{ge2018minimax} and \cite{dwork2014analyze}, we know that Step 1 keeps  the matrix $(\epsilon, \delta)$-differentially private. Thus, Algorithm 1 is $(\epsilon, \delta)$-differentially private due to the post-processing property of differential privacy \cite{dwork2006calibrating}.
\end{proof}

For the  matrix $\hat{\Sigma}$ in (\ref{eq:10}) after the first step of thresholding, we have the following key lemma.

\begin{lemma}\label{lemma:3}
For every fixed $1\leq i, j\leq p$, there exists a constant $C_1>0$ such that with probability at least $1-C_1p^{-\frac{9}{2}}$, the following holds:
\begin{equation}\label{eq:11a}
    |\hat{\sigma}_{ij}-\sigma_{ij}|\leq 4\min \{|\sigma_{ij}|,\gamma\sqrt{\frac{\log p}{n}}+\frac{4\sqrt{2\ln 1.25/\delta}\sqrt{\log p}}{n\epsilon}\}.
\end{equation}
\end{lemma}

\begin{proof}[Proof of Lemma \ref{lemma:3}]
Let
$\Sigma^*=(\sigma^*_{ij})_{1\leq i,j\leq p}$ and $N=(n_{ij})_{1\leq i,j\leq p}$. Define the event $A_{ij}=\{|\tilde{\sigma}_{ij}|>\gamma\sqrt{\frac{\log p}{n}}+\frac{4\sqrt{2\ln 1.25/\delta}\sqrt{\log p}}{n\epsilon}\}$. We have:
\begin{equation}\label{aeq:1}
    |\hat{\sigma}_{ij}-\sigma_{ij}|=|\sigma_{ij}|\cdot I(A^c_{ij})+ |\tilde{\sigma}_{ij}-\sigma_{ij}|\cdot I(A_{ij}).
\end{equation}
By the triangle inequality, it is easy to see that 
\begin{align*}
    A_{ij}&=\big\{|\tilde{\sigma}_{ij}-\sigma_{ij}+\sigma_{ij}|>\gamma\sqrt{\frac{\log p}{n}}+\frac{4\sqrt{2\ln 1.25/\delta}\sqrt{\log p}}{n\epsilon}\big\}\\
    & \subset \big\{|\tilde{\sigma}_{ij}-\sigma_{ij}|> \gamma\sqrt{\frac{\log p}{n}}+\frac{4\sqrt{2\ln 1.25/\delta}\sqrt{\log p}}{n\epsilon}-|\sigma_{ij}|\big\}
\end{align*}
and 
\begin{align*}
    A^c_{ij}&=\big \{|\tilde{\sigma}_{ij}-\sigma_{ij}+\sigma_{ij}|\leq\gamma\sqrt{\frac{\log p}{n}}+\frac{4\sqrt{2\ln 1.25/\delta}\sqrt{\log p}}{n\epsilon}\big\}\\
    &\subset \big\{|\tilde{\sigma}_{ij}-\sigma_{ij}|>|\sigma_{ij}|-( \gamma\sqrt{\frac{\log p}{n}}+\frac{4\sqrt{2\ln 1.25/\delta}\sqrt{\log p}}{n\epsilon})\big\}.
\end{align*}
Depending on the value of $\sigma_{ij}$, we have the following three cases.
\paragraph{\bf{Case 1}} $|\sigma_{ij}|\leq \frac{\gamma}{4}\sqrt{\frac{\log p}{n}}+\frac{\sqrt{2\log 1.25/\delta}\sqrt{\log p}}{n\epsilon}$. For this case, we have
\begin{multline}
    \mathbb{P}(A_{ij})\leq \mathbb{P}(|\tilde{\sigma}_{ij}-\sigma_{ij}|> \frac{3\gamma}{4}\sqrt{\frac{\log p}{n}}+\frac{3\sqrt{2\ln 1.25/\delta}\sqrt{\log p}}{n\epsilon}) \leq C_1p^{-\frac{9}{2}}+2p^{-\frac{9}{2}}.
\end{multline}
This is due to the followings:
\begin{align}
    &\mathbb{P}\big(|\tilde{\sigma}_{ij}-\sigma_{ij}|> \frac{3\gamma}{4}\sqrt{\frac{\log p}{n}}+\frac{3\sqrt{2\ln 1.25/\delta}\sqrt{\log p}}{n\epsilon}\big)\\
    &\leq  \mathbb{P}\big(|\sigma^*_{ij}-\sigma_{ij}|> \frac{3\gamma}{4}\sqrt{\frac{\log p}{n}}+\frac{3\sqrt{2\ln 1.25/\delta}\sqrt{\log p}}{n\epsilon})-|n_{ij}|\big) \label{aeq:2} \\
    &=\mathbb{P}\big(B_{ij}\bigcap \big\{\frac{3\sqrt{2\ln 1.25/\delta}\sqrt{\log p}}{n\epsilon})-|n_{ij}|>0\big\}\big)\\
    &+\mathbb{P}\big(B_{ij}\bigcap \big\{\frac{3\sqrt{2\ln 1.25/\delta}\sqrt{\log p}}{n\epsilon})-|n_{ij}|\leq 0\big\}\big) \label{aeq:3}\\
    &\leq \mathbb{P}(|\sigma^*_{ij}-\sigma_{ij}|> \frac{3\gamma}{4}\sqrt{\frac{\log p}{n}})+\mathbb{P}(\frac{2\sqrt{3\ln 1.25/\delta}\log p}{n\epsilon})\leq |n_{ij}|) \label{aeq:4}\\
    & \leq C_1P^{-\frac{9}{2}}+2p^{-\frac{9}{2}},
\end{align}
where  event $B_{ij}$ denotes $B_{ij}=\{|\sigma^*_{ij}-\sigma_{ij}|> \frac{3\gamma}{4}\sqrt{\frac{\log p}{n}}+\frac{2\sqrt{2\ln 1.25/\delta}\log p}{n\epsilon})-|n_{ij}|\}$,  and the last inequality is due to (\ref{eq:4}) and (\ref{eq:5}).

Thus by (\ref{aeq:1}), with probability at least $1-C_1p^{-\frac{9}{2}}-2p^{-\frac{9}{2}}$, we have
\begin{equation*}
    |\hat{\sigma}_{ij}-\sigma_{ij}|=|\sigma_{ij}|,
\end{equation*}
which satisfies (\ref{eq:11a}).

\paragraph{\bf{Case 2}}  
$|\sigma_{ij}|\geq 2\gamma\sqrt{\frac{\log p}{n}}+\frac{8\sqrt{2\ln 1.25/\delta}\sqrt{\log p}}{n\epsilon})$. For this case, we have
\begin{align*}
    \mathbb{P}(A_{ij}^c)\leq \mathbb{P}(|\tilde{\sigma}_{ij}-\sigma_{ij}|\geq \gamma\sqrt{\frac{\log p}{n}}+\frac{4\sqrt{2\ln 1.25/\delta}\sqrt{\log p}}{n\epsilon})\leq C_1p^{-8}+2p^{-8},
\end{align*}
where the proof is the same as (13-17). Thus, with probability at least $1-C_1p^{-\frac{9}{2}}-2p^{-8}$, we have
\begin{equation}
      |\hat{\sigma}_{ij}-\sigma_{ij}|=|\tilde{\sigma}_{ij}-\sigma_{ij}|.
\end{equation}
Also, by (\ref{eq:9}), (\ref{eq:11a}) also holds.

\paragraph{\bf{Case 3}} Otherwise, 
\begin{align*}
    \frac{\gamma}{4}\sqrt{\frac{\log p}{n}}+\frac{\sqrt{2\log 1.25/\delta}\sqrt{\log p}}{n\epsilon}\leq |\sigma_{ij}| \leq  2\gamma\sqrt{\frac{\log p}{n}}+\frac{8\sqrt{2\ln 1.25/\delta}\sqrt{\log p}}{n\epsilon}).
\end{align*}
For this case, we have 
\begin{equation}
      |\hat{\sigma}_{ij}-\sigma_{ij}|=|\sigma_{ij}| \text{ or }|\tilde{\sigma}_{ij}-\sigma_{ij}| .
\end{equation}
When $|\sigma_{ij}|\leq \gamma\sqrt{\frac{\log p}{n}}+\frac{4\sqrt{2\ln 1.25/\delta}\sqrt{\log p}}{n\epsilon}$,  we can see from (\ref{eq:9}) that with probability at least $1-2p^{-6}-C_1p^{-8}$, $$|\tilde{\sigma}_{ij}-\sigma_{ij}|\leq \gamma\sqrt{\frac{\log p}{n}}+\frac{4\sqrt{2\ln 1.25/\delta}\sqrt{\log p}}{n\epsilon}\leq 4|\sigma_{ij}|.$$
Thus, $(\ref{eq:11a})$ also holds.

Otherwise when  $|\sigma_{ij}|\leq \gamma\sqrt{\frac{\log p}{n}}+\frac{4\sqrt{2\ln 1.25/\delta}\sqrt{\log p}}{n\epsilon}$, $(\ref{eq:11a})$ also holds. Thus, Lemma 3 is true.
\end{proof}

By Lemma \ref{lemma:3}, we have the following upper bound on the $\ell_2$-norm error of $\Sigma^+$.

\begin{theorem}\label{thm:2}
The output $\Sigma^+$ of Algorithm \ref{alg:1} satisfies:
\begin{align}
   \mathbb{E}\|\hat{\Sigma}-\Sigma\|^2_2=O(\frac{s^2\log p\log \frac{1}{\delta}}{n\epsilon^2}) \label{eq:19},
\end{align}
where the expectation is taken over the coins of the Algorithm and the randomness of $\{x_1, x_2, \cdots, x_n\}$.
\end{theorem}

\begin{proof}[Proof of Theorem \ref{thm:2}]
We first show that $\|\Sigma^+-\Sigma\|_2\leq 2\|\hat{\Sigma}-\Sigma\|_2$. This is due to the following 
\begin{align}
   & \|\Sigma^+-\Sigma\|_2\leq \|\Sigma^+-\hat{\Sigma}\|_2+\|\hat{\Sigma}-\Sigma\|_2\leq \max_{i: \lambda_i\leq 0}|\lambda_i|+\|\hat{\Sigma}-\Sigma\|_2 \nonumber\\
   &\leq \max_{i: \lambda_i\leq 0}|\lambda_i-\lambda_i(\Sigma)|+\|\hat{\Sigma}-\Sigma\|_2\leq 2\|\hat{\Sigma}-\Sigma\|_2, \nonumber
\end{align}
where the third inequality is due to the fact that $\Sigma$ is positive semi-definite.

This means that we only need to bound $\|\hat{\Sigma}-\Sigma\|_2$. Since $\hat{\Sigma}-\Sigma$ is symmetric, we know that $\|\hat{\Sigma}-\Sigma\|_2\leq  \|\hat{\Sigma}-\Sigma\|_1$ \cite{golub2012matrix}. Thus, it suffices to prove that the bound in (\ref{eq:19}) holds for $\|\hat{\Sigma}-\Sigma\|_1$.

We define event $E_{ij}$ as 
\begin{equation}\label{aeq:5}
    E_{ij}=\{    |\hat{\sigma}_{ij}-\sigma_{ij}|\leq 4\min \{|\sigma_{ij}|,\gamma\sqrt{\frac{\log p}{n}}+\frac{4\sqrt{2\ln 1.25/\delta}\sqrt{\log p}}{n\epsilon}\} \}.
\end{equation}
Then, by Lemma \ref{lemma:3}, we have $\mathbb{P}(E_{ij})\geq 1-2C_1p^{-\frac{9}{2}}$.

Let $D=(d_{ij})_{1\leq i,j\leq p}$, where $d_{ij}=(\hat{\sigma}_{ij}-\sigma_{ij})\cdot I(E_{ij}^c)$. Then, we have 
\begin{align}
    &\|\hat{\Sigma}-\Sigma\|_1^2\leq  \|\hat{\Sigma}-\Sigma-D+D\|_1^2 \nonumber \\
    &\leq 2\|\hat{\Sigma}-\Sigma-D\|_1^2+2\|D\|_1^2 \nonumber\\
    &\leq 4(\sup_{j}\sum_{i\neq j}|\hat{\sigma}_{ij}-\sigma_{ij}|I(E_{ij}))^2+2\|D\|_1^2+O(\frac{\log p\log \frac{1}{\delta}}{n\epsilon^2}). \label{aeq:6}
\end{align}
We first bound the first term of (\ref{aeq:6}). By the definition of $E_{ij}$ and Lemma 3, we can upper bounded it by 
\begin{align}
&(\sup_{j}\sum_{i\neq j}|\hat{\sigma}_{ij}-\sigma_{ij}|I(E_{ij}))^2 \nonumber \\
&\leq 16(\sup_{j}\sum_{i\neq j}\min \{|\sigma_{ij}|,\gamma\sqrt{\frac{\log p}{n}}+\frac{4\sqrt{2\ln 1.25/\delta}\sqrt{\log p}}{n\epsilon}\})^2 \nonumber\\
&\leq 16 s^2 (\gamma\sqrt{\frac{\log p}{n}}+\frac{4\sqrt{2\ln 1.25/\delta}\sqrt{\log p}}{n\epsilon})^2 \nonumber\\
&\leq O(s^2\frac{\log p \log 1/\delta}{n\epsilon^2}),
\end{align}
where the second inequality is due to the assumption that at most $s$ elements of $(\sigma_{ij})_{i\neq j}$ are non-zero. 

For the second term in (\ref{aeq:6}), we have 
\begin{align}
    &\mathbb{E}\|D\|_1^2\leq p\mathbb \sum_{ij}d_{ij}^2 \nonumber=p\mathbb{E} \sum_{ij}[(\hat{\sigma}_{ij}-\sigma_{ij})^2I(E_{ij}^c\bigcap \{\hat{\sigma}_{ij}=\tilde{\sigma}_{ij}\})\nonumber\\
    &+(\hat{\sigma}_{ij}-\sigma_{ij})^2I(E_{ij}^c\bigcap \{\hat{\sigma}_{ij}=0\})]\nonumber\\
    &=p\mathbb{E}\sum_{ij}[(\tilde{\sigma}_{ij}-\sigma_{ij})^2I(E_{ij}^c)+p\sum_{ij}\mathbb{E}\sigma_{ij}^2I(E_{ij}^c\bigcap \{\hat{\sigma}_{ij}=0\})].\label{aeq:7}
\end{align}
For the first term in (\ref{aeq:7}),  we have 
\begin{align}
    &p\sum_{ij}\mathbb{E}\{(\tilde{\sigma}_{ij}-\sigma_{ij})^2I(E_{ij}^c)\}\leq p\sum_{ij}[\mathbb{E}(\tilde{\sigma}_{ij}-\sigma_{ij})^6]^{\frac{1}{3}}\mathbb{P}^{\frac{2}{3}}(E_{ij}^c) \label{aeq:8}\\
    &\leq Cp \cdot p^2 \frac{1}{n\epsilon^2 } p^{-3}=O(\frac{1}{n\epsilon^2}), \nonumber
\end{align}
where the first inequality is due to H\"{o}lder inequality and the second inequality is due to the fact that  $\mathbb{E}(\tilde{\sigma}_{ij}-\sigma_{ij})^8\leq C_3 [\mathbb{E}(\sigma^*_{ij}-\sigma_{ij})^8+\mathbb{E}n_{ij}^8]$. Since $n_{ij}$ is a Gaussian distribution, we have \cite{papoulis1965probability} $\mathbb{E}n_{ij}^8\leq C_4 \sigma_1^8=O(\frac{1}{n\epsilon})$. For the first term $\mathbb{E}(\sigma^*_{ij}-\sigma_{ij})^8$, since $x_i$ is sampled from a  sub-Gaussian distribution (\ref{eq:2}),  by Whittle Inequality (Theorem 2 in \cite{whittle1960bounds} or \cite{cai2012optimal}), the quadratic form $\sigma^*_{ij}$ satisfies $\mathbb{E}(\sigma^*_{ij}-\sigma_{ij})^8\leq C_5 \frac{1}{n}$ for some positive constant $C_5>0$. 

For the second term of (\ref{aeq:7}), we have 
\begin{align}
    &p\sum_{ij}\mathbb{E}\sigma_{ij}^2I(E_{ij}^c\bigcap \{\hat{\sigma}_{ij}=0\}) \nonumber \\
    &=p\sum_{ij}\mathbb{E}\sigma_{ij}^2 I(|\sigma_{ij}|>4\gamma\sqrt{\frac{\log p}{n}}+\frac{16\sqrt{2\ln 1.25/\delta}\sqrt{\log p}}{n\epsilon})\nonumber \\
    &\times I(|\tilde{\sigma}_{ij}|\leq \gamma\sqrt{\frac{\log p}{n}}+\frac{4\sqrt{2\ln 1.25/\delta}\sqrt{\log p}}{n\epsilon})\nonumber\\
    &\leq p\sum_{ij}\mathbb{E}\sigma_{ij}^2I(|\sigma_{ij}|>4\gamma\sqrt{\frac{\log p}{n}}+\frac{16\sqrt{2\ln 1.25/\delta}\sqrt{\log p}}{n\epsilon})\nonumber \\
    &\times I(|\sigma_{ij}|-|\tilde{\sigma}_{ij}-\sigma_{ij}|\leq \gamma\sqrt{\frac{\log p}{n}}+\frac{4\sqrt{2\ln 1.25/\delta}\sqrt{\log p}}{n\epsilon})\nonumber\\
    &\leq p\sum_{ij}\sigma_{ij}^2\mathbb{E}I(|\sigma_{ij}|>4\gamma\sqrt{\frac{\log p}{n}}+\frac{16\sqrt{2\ln 1.25/\delta}\sqrt{\log p}}{n\epsilon})I(|\tilde{\sigma}_{ij}-\sigma_{ij}|\geq  \frac{3}{4}|\sigma_{ij}|)\nonumber\\
    &\leq p\sum_{ij}\sigma_{ij}^2\mathbb{E}I(|\sigma_{ij}|>4\gamma\sqrt{\frac{\log p}{n}}+\frac{16\sqrt{2\ln 1.25/\delta}\sqrt{\log p}}{n\epsilon})I(|\sigma^*_{ij}-\sigma_{ij}|+|n_{ij}|\geq  \frac{3}{4}|\sigma_{ij}|)\nonumber\\
    &\leq p\sum_{ij}\sigma_{ij}^2\mathbb{P}\big(\big\{|\sigma^*_{ij}-\sigma_{ij}|\geq \frac{3}{4}|\sigma_{ij}|-|n_{ij}|\big\}\bigcap \big\{|\sigma_{ij}|>4\gamma\sqrt{\frac{\log p}{n}}+\frac{16\sqrt{2\ln 1.25/\delta}\sqrt{\log p}}{n\epsilon}\big\}\big)\nonumber\\
    &  =p\sum_{ij}\sigma_{ij}^2\mathbb{P}\big(\big\{|\sigma^*_{ij}-\sigma_{ij}|\geq \frac{3}{4}|\sigma_{ij}|-|n_{ij}|\big\}\bigcap \big\{|n_{ij}|\leq \frac{1}{4}|\sigma_{ij}|\big\}\bigcap \nonumber \\
    &\big\{|\sigma_{ij}|>4\gamma\sqrt{\frac{\log p}{n}}+\frac{16\sqrt{2\ln 1.25/\delta}\sqrt{\log p}}{n\epsilon}\big\}\big) +p\sum_{ij}\sigma_{ij}^2\mathbb{P}\big(\big\{|\sigma^*_{ij}-\sigma_{ij}|\geq \frac{3}{4}|\sigma_{ij}|-|n_{ij}|\big\}\nonumber \\
    &\bigcap \big\{|n_{ij}|\geq \frac{1}{4}|\sigma_{ij}|\big\} 
    \bigcap \big\{|\sigma_{ij}|>4\gamma\sqrt{\frac{\log p}{n}}+\frac{16\sqrt{2\ln 1.25/\delta}\sqrt{\log p}}{n\epsilon}\big\}\big) \label{aeq:9}\\
    &\leq p\sum_{ij}\sigma_{ij}^2\mathbb{P}\big(\big\{|\sigma^*_{ij}-\sigma_{ij}|\geq \frac{1}{2}|\sigma_{ij}|\big\}\bigcap \big\{|\sigma_{ij}|>4\gamma\sqrt{\frac{\log p}{n}}+\frac{16\sqrt{2\ln 1.25/\delta}\sqrt{\log p}}{n\epsilon}\big\}\big)\nonumber \\
    &+p\sum_{ij}\sigma^2_{ij}\mathbb{P}\big(\big\{|n_{ij}|\geq \frac{1}{4}|\sigma_{ij}|\big\}\bigcap \big\{|\sigma_{ij}|>4\gamma\sqrt{\frac{\log p}{n}}+\frac{16\sqrt{2\ln 1.25/\delta}\sqrt{\log p}}{n\epsilon}\big\}\big) \label{aeq:10}.
\end{align}
For the second term of (\ref{aeq:10}), by Lemmas 1 and 2  we have 
\begin{align}
& p\sum_{ij}\sigma^2_{ij}\mathbb{P}(\{|n_{ij}|\geq \frac{1}{4}|\sigma_{ij}|\}\bigcap \{|\sigma_{ij}|>4\gamma\sqrt{\frac{\log p}{n}}+\frac{16\sqrt{2\ln 1.25/\delta}\sqrt{\log p}}{n\epsilon}\}) \nonumber \\
&\leq p\sum_{ij}\sigma^2_{ij}\mathbb{P}(|n_{ij}|\geq \gamma\sqrt{\frac{\log p}{n}}+ \frac{4\sqrt{2\ln 1.25/\delta}\log p}{n\epsilon}\})\mathbb{P}(|n_{ij}|>\frac{1}{4}\sigma_{ij}) \nonumber \\
&\leq Cp\sum_{ij}\sigma_{ij}^2\cdot \exp(-\frac{(\gamma\sqrt{\frac{\log p}{n}}+4\sigma_1\sqrt{\log p})^2}{2\sigma_1^2} )\exp(-\frac{\sigma_{ij}^2}{32\sigma_1^2}) \nonumber \\
&\leq C\sigma_1^2 p\cdot p^2\exp(-\frac{\gamma^2 \log p}{2n\sigma_1^2})p^{-8}\\
&\leq C\sigma_1^2 p^{-5}\frac{2n\sigma_1^2}{\gamma^2\log p} = O(\frac{\log 1/\delta}{n\epsilon^2}).
\end{align}
For the first term of (\ref{aeq:10}), by Lemma 2 we have 
\begin{align}
    &p\sum_{ij}\sigma_{ij}^2\mathbb{P}(\{|\sigma^*_{ij}-\sigma_{ij}|\geq \frac{1}{2}|\sigma_{ij}|\}\bigcap \{|\sigma_{ij}|\geq 4\gamma\sqrt{\frac{\log p}{n}}\})\nonumber \\
    &\leq \frac{p}{n}\sum_{ij}n\sigma_{ij}^2\exp(-n\frac{2\sigma_{ij}^2}{\gamma^2})I(|\sigma_{ij}|\geq 4\gamma\sqrt{\frac{\log p}{n}})\nonumber \\
    &\leq \frac{p}{n}\sum_{ij}[n\sigma_{ij}^2\exp(-n\frac{\sigma_{ij}^2}{\gamma^2})] \exp(-n\frac{\sigma_{ij}^2}{\gamma^2}) I(|\sigma_{ij}|\geq 4\gamma\sqrt{\frac{\log p}{n}})\nonumber \\
    &\leq C\frac{p^3}{n}p^{-16}=O(\frac{1}{n}).
    \end{align}
Thus in total, we have $\mathbb{E}\|D\|_1^2=O(\frac{\log 1/\delta}{n\epsilon^2})$. This means that  $\mathbb{E}\|\hat{\Sigma}-\Sigma\|_1^2=O(\frac{s^2\log p\log 1/\delta}{n\epsilon^2})$, which completes the proof.
\end{proof}

\begin{corollary}\label{cor:1}
For any $1\leq w\leq \infty$, the matrix $\hat{\Sigma}$ in (\ref{eq:10}) after the first step of thresholding satisfies 
\begin{equation}\label{eq:20}
    \|\hat{\Sigma}-\Sigma\|^2_w\leq O(s^2\frac{\log p\log \frac{1}{\delta}}{n\epsilon^2}),
\end{equation}
where the $w$-norm of any matrix $A$ is defined as $\|A\|_w=\sup \frac{\|Ax\|_w}{\|x\|_w}$. Specifically, for a matrix $A=(a_{ij})_{1\leq i,j\leq p}$, 
$\|A\|_1=\sup_{j}\sum_{i}|a_{ij}|$ is the maximum absolute column sum, and 
$\|A\|_\infty=\sup_{i}\sum_{j}|a_{ij}|$
is the maximum absolute row sum.
\end{corollary}
Comparing the bound in the above corollary with the optimal minimax rate  $\Theta(\frac{s^2\log p}{n})$ in \cite{cai2012optimal} for the non-private case, we can see that   
the impact of the  differential privacy is to make the number of efficient sample from $n$ to $n\epsilon^2$. It is an open problem to determine whether  the bound in Theorem \ref{thm:2} is tight.

\begin{proof}[Proof of Corollary \ref{cor:1}]
By Riesz-Thorin interpolation theorem \cite{dunford1958linear}, we have $$\|A\|_w\leq \max \{ \|A\|_1, \|A\|_2, \|A\|_\infty\}$$ for any matrix $A$ and any $1\leq w\leq \infty.$ Since $\Sigma^+-\Sigma$ is a symmetric matrix, we have $\|\Sigma^+-\Sigma\|_2\leq \|\Sigma^+-\Sigma\|_1$ and $\|\Sigma^+-\Sigma\|_1=\|\Sigma^+-\Sigma\|_\infty$. Thus, by the proof of Theorem \ref{thm:2} we get this corollary.
\end{proof}

\subsection{Extension to Local Differential Privacy}

One  advantage of our Algorithm \ref{alg:1} is that it can be easily  extended to the locally differentially private (LDP) model. 

\paragraph{Differential privacy in the local model.} In LDP, we have a data universe $\mathcal{D}$,  $n$ players with each holding  a private data record $x_i\in \mathcal{D}$, and a server that is in charge of coordinating the protocol. An LDP protocol proceeds in $T$ rounds. In each round, the server sends a message, which sometime is called a query, to a subset of the players, requesting them to run a particular algorithm. Based on the queries, each player $i$ in the subset selects an algorithm $Q_i$, run it on her data, and sends the output back to the server.

\begin{definition}\cite{dwangnips18}\label{def:1}
An algorithm $Q$ is $(\epsilon, \delta)$-locally differentially private (LDP) if for all pairs $x,x'\in \mathcal{D}$, and for all events $E$ in the output space of $Q$, we have $\text{Pr}[Q(x)\in E]\leq e^{\epsilon}\text{Pr}[Q(x')\in E]+\delta.$
A multi-player protocol is $\epsilon$-LDP if for all possible inputs and runs of the protocol, the transcript of player i's interaction with the server is $\epsilon$-LDP. If $T=1$, we say that the protocol is $(\epsilon, \delta)$ non-interactive LDP.
\end{definition}
	\begin{algorithm}[h]
		\caption{LDP-Thresholding}
		$\mathbf{Input}$: $\epsilon, \delta$ are privacy parameters, $\{x_1,x_2,\cdots, x_n\}\sim P\in \mathcal{P}(\sigma^2, s)$.
		\begin{algorithmic}[1]
				\FOR {Each $i\in [n]$}
	    \STATE Denote $\tilde{x}_i\tilde{x}_i^T=x_ix_i^T+z_i$, where $z_i \in \mathbb{R}^{p\times p}$ is a symmetric matrix with its upper triangle ( including the diagonal) being i.i.d samples from 
	    $\mathcal{N}(0, \sigma^2)$; here $\sigma^2=\frac{2\ln(1.25/\delta)}{\epsilon^2}$, and each lower triangle entry is copied from its upper triangle counterpart. 
	    	\ENDFOR
		   \STATE Compute $\tilde{\Sigma}=(\tilde{\sigma}_{ij})_{1\leq i,j\leq p}=\frac{1}{n}\sum_{i=1}^n\tilde{x}_i\tilde{x}_i^T,$
	    \STATE Define the thresholding estimator $\hat{\Sigma}=(\hat{\sigma}_{ij})_{1\leq i,j\leq n}$ as
		\begin{equation}\label{eq:11}
		    \hat{\sigma}_{ij}=\tilde{\sigma}_{ij}\cdot I[|\tilde{\sigma}_{ij}|>\gamma\sqrt{\frac{\log p}{n}}+\frac{4\sqrt{2\ln 1.25/\delta}\sqrt{\log p}}{\sqrt{n}\epsilon}].
		\end{equation}
		\STATE Let the eigen-decomposition of $\hat{\Sigma}$ as $\hat{\Sigma}=\sum_{i=1}^p\lambda_iv_iv_i^T$. Let $\lambda^+=\max\{\lambda_i, 0\}$ be the positive part of $\lambda_i$, then define $\Sigma^+=\sum_{i=1}^p\lambda^+ v_iv_i^T$.
		\RETURN $\Sigma^+$.
		\end{algorithmic}\label{alg:2}
	\end{algorithm}

Inspired by Algorithm \ref{alg:1}, it is easy to extend our DP algorithm to the LDP model. The idea is that each $X_i$ perturbs its covariance and aggregates the noisy version of covariance, see Algorithm \ref{alg:2} for detail. 

The following theorem shows that the error bound of the output in Algorithm \ref{alg:2} is the same as the the bound in Theorem \ref{thm:2} asymptotically, whose proof is almost the same as in Theorem \ref{thm:2}.

\begin{theorem}\label{thm:3}
The output $\Sigma^+$ of Algorithm \ref{alg:2} satisfies:
\begin{align}
   \mathbb{E}\|\hat{\Sigma}-\Sigma\|^2_2=O(\frac{s^2\log p\log \frac{1}{\delta}}{n\epsilon^2}),
\end{align}
where the expectation is taken over the coins of the Algorithm and the randomness of $\{x_1, x_2, \cdots, x_n\}$. Moreover,   $\hat{\Sigma}$ in (\ref{eq:11}) satisfies $\|\hat{\Sigma}-\Sigma\|_{w}^2=O(\frac{s\log p\log \frac{1}{\delta}}{n\epsilon^2}) $.
\end{theorem}

\section{Experiments}

In this section, we evaluate the performance of  Algorithm \ref{alg:1} and \ref{alg:2} practically on  synthetic datasets.

\paragraph{Data Generation} 
We first generate a symmetric sparse matrix $\tilde{U}$ with the sparsity ratio $sr$, that is, there are $sr\times p \times p$ non-zero entries of the matrix.  Then, we let $U= \tilde{U}+\lambda I_{p}$ for some constant $\lambda$ to make $U$ positive semi-definite and then scale it to $U=\frac{U}{c}$ by some constant $c$ which makes the norm of samples less than 1 (with high probability)\footnote{Although the distribution is not bounded by 1, actually, as we see from previous section, we can obtain the same result as long as the $\ell_2$ norm of the samples is bounded by 1.}. Finally, we sample $\{x_1, \cdots, x_n\}$ from the multivariate Gaussian distribution $\mathcal{N}(0, U)$. In this paper, we will use set $\lambda=50$ and $c=200$.

\paragraph{Experimental Settings} To measure the performance, we compare the $\ell_1$ and $\ell_2$ norm of relative error, respectively. That is, $\frac{\|\Sigma^+-U\|_2}{\|U\|_2}$ or $\frac{\|\Sigma^+-U\|_1}{\|U\|_1}$ with the sample size $n$ in three different settings: 1) we set $p=100$, $\epsilon=1$, $\delta=\frac{1}{n}$ and change the sparse ratio $sr=\{0.1, 0.2, 0.3, 0.5\}$. 2) We set $\epsilon=1$, $\delta=\frac{1}{n}$, $sr=0.2$, and let the dimensionality $p$ vary in $\{50,100, 200, 500\}$. 3) We fix $p=200$, $\delta=\frac{1}{n}$, $sr=0.2$ and change the privacy level as $\epsilon=\{0.1,  0.5, 1, 2\}$.
 We run each experiment 20 times and take the average error as the final one. 
 
 \paragraph{Experimental Results} Figure \ref{fig:1} and \ref{fig:2} are the results of DP-Thresholding (Algorithm \ref{alg:1}) with $\ell_2$  and $\ell_1$ relative error, respectively. Figure \ref{fig:3} and \ref{fig:4} are the results of LDP-Thresholding (Algorithm \ref{alg:2}) with $\ell_2$  and $\ell_1$ relative error, respectively. From  the figures we can see that: 1) if the sparsity ratio is large {\em i.e.,} the underlying covairance matrix is more dense, the relative error will be larger, this is due to the  fact showed in Theorem \ref{thm:2} and \ref{thm:3}  that the error depends on the sparsity $s$. 2) The dimensionality only slightly affects the relative error. That is, even if we double the value of $p$, the error increases only slightly. 
 This is consistent with our theoretical analysis in Theorem \ref{thm:2} and \ref{thm:3} which says that the error of our private estimators is only logarithmically depending on $p$ ({\em i.e.,} $\log p$).
 3) With the privacy parameter $\epsilon$ increases (which means more private), the error will become larger. This has also been showed in previous theorems.
 
 In summary, all the experimental results support our theoretical analysis.

\begin{figure*}
    \centering
    \includegraphics[width=1\textwidth,height=0.20\textheight]{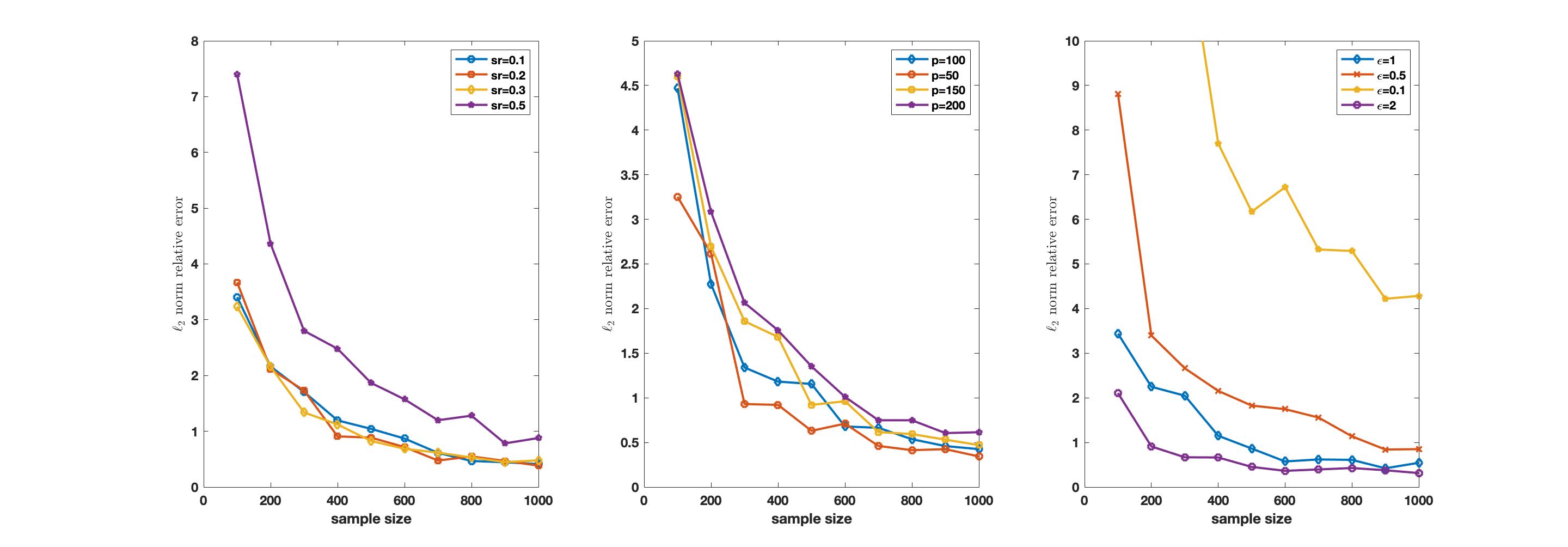}
    \caption{Experiment results of Algorithm \ref{alg:1} for $\ell_2$-norm relative error. The left one is for different sparsity levels, the middle one is for different dimensionality $p$, and the right one is for different privacy level $\epsilon$.}
    \label{fig:1}
\end{figure*}

\begin{figure*}
    \centering
    \includegraphics[width=1\textwidth,height=0.20\textheight]{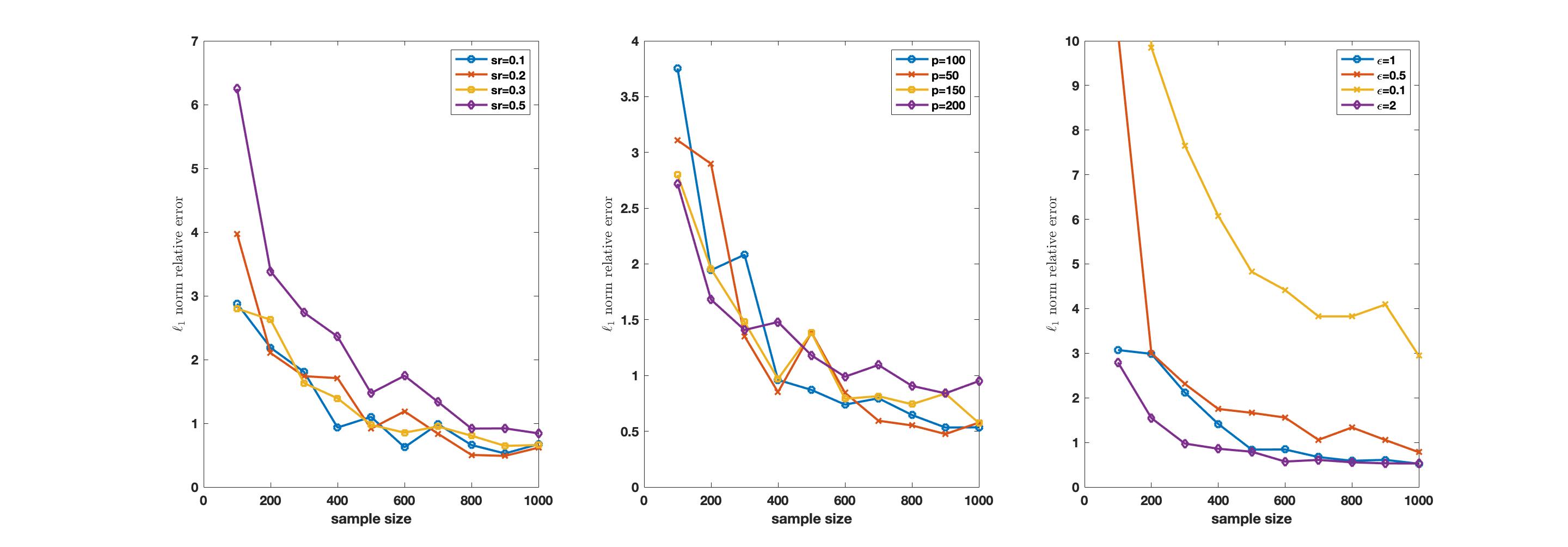}
    \caption{Experiment results of Algorithm \ref{alg:1} for $\ell_1$-norm relative error. The left one is for different sparsity levels, the middle one is for different dimensionality $p$, and the right one is for different privacy level $\epsilon$.}
    \label{fig:2}
\end{figure*}

\begin{figure*}
    \centering
    \includegraphics[width=1\textwidth,height=0.20\textheight]{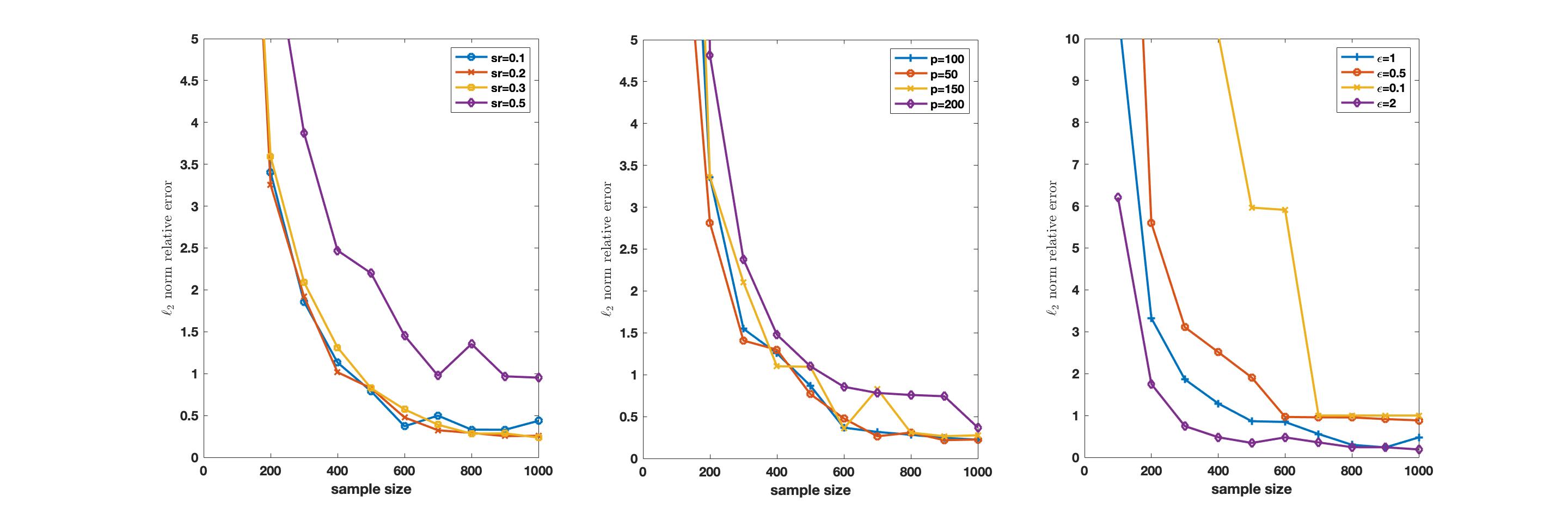}
    \caption{Experiment results of Algorithm \ref{alg:2} for $\ell_2$-norm relative error. The left one is for different sparsity levels,  the middle one is for different dimensionality $p$, and the right one is for different privacy level $\epsilon$. }
    \label{fig:3}
\end{figure*}
\begin{figure*}
    \centering
    \includegraphics[width=1\textwidth,height=0.20\textheight]{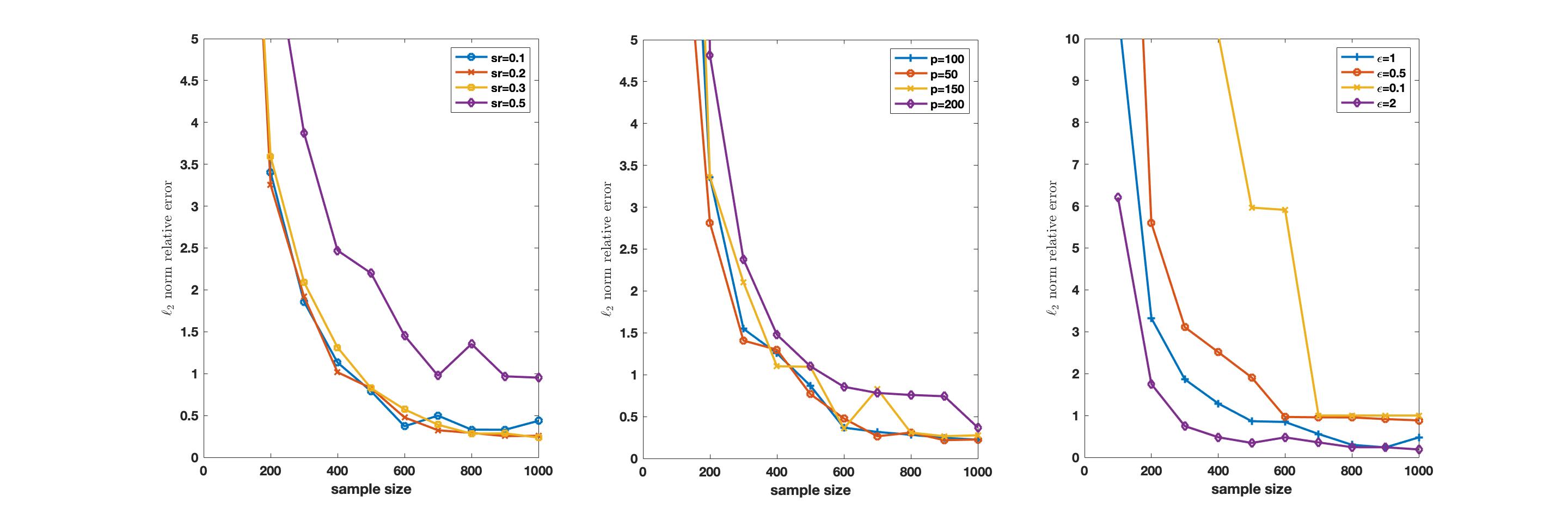}
    \caption{Experiment results of Algorithm \ref{alg:2} for $\ell_1$-norm relative error. The left one is for different sparsity levels, the middle one is for different dimensionality $p$, and the right one is for different privacy level $\epsilon$.  }
    \label{fig:4}
\end{figure*}

\section{Conclusion and Discussion}

In the paper, we study the problem of estimating the sparse covariance matrix of a bounded sub-Gaussian distribution under differential privacy model and propose a method called DP-Threshold, which achieves a non-trivial error bound and can be easily extended to the local model. Experiments on synthetic datasets yield consistent results with  the theoretical analysis.

There are still some open problems for this problem. Firstly, although the thresholding method can achieve non-trivial error bound for our private estimator,
in practice it is hart to find the best threshold. Thus, an open problem is how to get the best threshold.
Secondly, as mentioned in the related work section, there are many recent results on  private Gaussian estimation, which may make the $\ell_2$ norm of the samples greater than 1. Thus, it is an interesting problem to extend our method to a general Gaussian distribution.

\section*{References}

\bibliography{tcs_2019}

\begin{thebibliography}{10}
\expandafter\ifx\csname url\endcsname\relax
  \def\url#1{\texttt{#1}}\fi
\expandafter\ifx\csname urlprefix\endcsname\relax\def\urlprefix{URL }\fi
\expandafter\ifx\csname href\endcsname\relax
  \def\href#1#2{#2} \def\path#1{#1}\fi

\bibitem{dwork2006calibrating}
C.~Dwork, F.~McSherry, K.~Nissim, A.~Smith, Calibrating noise to sensitivity in
  private data analysis, in: Theory of Cryptography Conference, Springer, 2006,
  pp. 265--284.

\bibitem{uber}
J.~Near, Differential privacy at scale: Uber and berkeley collaboration, in:
  Enigma 2018 (Enigma 2018), {USENIX} Association, Santa Clara, CA, 2018.

\bibitem{google}
{\'U}.~Erlingsson, V.~Pihur, A.~Korolova, Rappor: Randomized aggregatable
  privacy-preserving ordinal response, in: Proceedings of the 2014 ACM SIGSAC
  conference on computer and communications security, ACM, 2014, pp.
  1054--1067.

\bibitem{apple}
J.~Tang, A.~Korolova, X.~Bai, X.~Wang, X.~Wang, Privacy loss in apple's
  implementation of differential privacy on macos 10.12, CoRR abs/1709.02753.
\newblock \href {http://arxiv.org/abs/1709.02753} {\path{arXiv:1709.02753}}.

\bibitem{dwangppml18}
D.~Wang, A.~Smith, J.~Xu, High dimensional sparse linear regression under local
  differential privacy: Power and limitations, 2018 NIPS workshop in
  Privacy-Preserving Machine Learning.

\bibitem{duchi2018right}
J.~C. Duchi, F.~Ruan, The right complexity measure in locally private
  estimation: It is not the fisher information, arXiv preprint
  arXiv:1806.05756.

\bibitem{ullman2018tight}
J.~Ullman, Tight lower bounds for locally differentially private selection,
  arXiv preprint arXiv:1802.02638.

\bibitem{ge2018minimax}
J.~Ge, Z.~Wang, M.~Wang, H.~Liu, Minimax-optimal privacy-preserving sparse pca
  in distributed systems, in: International Conference on Artificial
  Intelligence and Statistics, 2018, pp. 1589--1598.

\bibitem{dwangglobalsip18}
D.~Wang, M.~Huai, J.~Xu, Differentially private sparse inverse covariance
  estimation, in: 2018 {IEEE} Global Conference on Signal and Information
  Processing, GlobalSIP 2018, Anaheim, CA, USA, November 26-29, 2018.

\bibitem{kamath2018privately}
G.~Kamath, J.~Li, V.~Singhal, J.~Ullman, Privately learning high-dimensional
  distributions, arXiv preprint arXiv:1805.00216.

\bibitem{joseph2018locally}
M.~Joseph, J.~Kulkarni, J.~Mao, Z.~S. Wu, Locally private gaussian estimation,
  arXiv preprint arXiv:1811.08382.

\bibitem{karwa2017finite}
V.~Karwa, S.~Vadhan, Finite sample differentially private confidence intervals,
  arXiv preprint arXiv:1711.03908.

\bibitem{gaboardi2018locally}
M.~Gaboardi, R.~Rogers, O.~Sheffet, Locally private mean estimation: Z-test and
  tight confidence intervals, arXiv preprint arXiv:1810.08054.

\bibitem{kareemppml18}
K.~Amin, T.~Dick, A.~Kulesza, A.~M. Medina, S.~Vassilvitskii, Private
  covariance estimation via iterative eigenvector sampling, 2018 NIPS workshop
  in Privacy-Preserving Machine Learning.

\bibitem{dwork2014analyze}
C.~Dwork, K.~Talwar, A.~Thakurta, L.~Zhang, Analyze gauss: optimal bounds for
  privacy-preserving principal component analysis, in: Proceedings of the 46th
  Annual ACM Symposium on Theory of Computing, ACM, 2014, pp. 11--20.

\bibitem{cai2012optimal}
T.~T. Cai, H.~H. Zhou, et~al., Optimal rates of convergence for sparse
  covariance matrix estimation, The Annals of Statistics 40~(5) (2012)
  2389--2420.

\bibitem{tao2012topics}
T.~Tao, Topics in random matrix theory, Vol. 132, American Mathematical Soc.,
  2012.

\bibitem{tropp2015introduction}
J.~A. Tropp, et~al., An introduction to matrix concentration inequalities,
  Foundations and Trends{\textregistered} in Machine Learning 8~(1-2) (2015)
  1--230.

\bibitem{bickel2008covariance}
P.~J. Bickel, E.~Levina, et~al., Covariance regularization by thresholding, The
  Annals of Statistics 36~(6) (2008) 2577--2604.

\bibitem{golub2012matrix}
G.~H. Golub, C.~F. Van~Loan, Matrix computations, Vol.~3, JHU Press, 2012.

\bibitem{papoulis1965probability}
A.~Papoulis, Probability, random variables, and stochastic processes.

\bibitem{whittle1960bounds}
P.~Whittle, Bounds for the moments of linear and quadratic forms in independent
  variables, Theory of Probability \& Its Applications 5~(3) (1960) 302--305.

\bibitem{dunford1958linear}
N.~Dunford, J.~T. Schwartz, Linear operators part I: general theory, Vol.~7,
  Interscience publishers New York, 1958.

\bibitem{dwangnips18}
D.~Wang, M.~Gaboardi, J.~Xu, \href{http://arxiv.org/abs/1802.04085}{Empirical
  risk minimization in non-interactive local differential privacy revisited},
  Advances in Neural Information Processing Systems 31: Annual Conference on
  Neural Information Processing Systems 2018, 3-8 December 2018, Montreal, QC,
  {Canada}\href {http://arxiv.org/abs/1802.04085} {\path{arXiv:1802.04085}}.
\newline\urlprefix\url{http://arxiv.org/abs/1802.04085}

\end{thebibliography}

\end{document}